%% file: SatAttnTC0-v2.tex
\documentclass[11pt,a4paper]{article}
\usepackage{times,latexsym}
\usepackage{url}
\usepackage[T1]{fontenc}

\usepackage[acceptedWithA]{tacl2021v1}

\usepackage{xspace,mfirstuc,tabulary}

\newif\iftaclinstructions
\taclinstructionsfalse 
\iftaclinstructions
\renewcommand{\confidential}{}
\renewcommand{\anonsubtext}{(No author info supplied here, for consistency with
TACL-submission anonymization requirements)}
\newcommand{\instr}
\fi

\iftaclpubformat 

\else

\fi

\input{macros}

\title{
\new{Saturated Transformers are Constant-Depth Threshold Circuits}
}

\author{
    William Merrill\footnotemark[1]\;\;\footnotemark[2]
    \quad Ashish Sabharwal \footnotemark[1]
    \quad Noah A. Smith\footnotemark[1]\;\;\footnotemark[3] \\
    \footnotemark[1]\;\;Allen Institute for AI\quad
    \footnotemark[2]\;\;New York University\quad
    \footnotemark[3]\;\;University of Washington \\
    {\texttt{willm@nyu.edu}\;\; \texttt{\{ashishs,noah\}@allenai.org}}
}

\date{\today}

\begin{document}
\maketitle
\begin{abstract}
Transformers have become a standard neural network architecture for many NLP problems, motivating theoretical analysis of their power in terms of formal languages. Recent work has shown that transformers with \emph{hard} attention are quite limited in power \citep{hahn-2020-theoretical}, \new{as they} can be simulated by constant-depth \new{AND/OR} circuits \citep{angluin2021}. \new{However,} hard attention is a strong assumption, which may complicate the relevance of these results \new{in practice}. In this work, we analyze the circuit complexity of transformers with \emph{saturated} attention: a generalization of hard attention that more closely captures the attention patterns learnable in practical transformers. We first \new{show} that saturated transformers transcend the known limitations of hard-attention transformers. We then \new{prove} saturated transformers with floating-point values can be simulated by constant-depth \new{threshold} circuits\new{, giving the class $\TC^0$ as an upper bound on the formal languages they recognize.}
\end{abstract}

\section{Introduction}

Opening the ``black box'' \citep{blackboxnlp-2020-blackboxnlp} of the representations within neural networks is an important step towards building systems with robust and interpretable behavior. In NLP, one part of this question is analyzing the languages networks can model, and the mechanisms they use to represent linguistic structure and dependencies.


One path toward this goal is via formal analysis of specific network architectures \citep{merrill2021formal}; for example, recurrent neural networks (RNNs).
Due to their autoregressive formulation, formal linguistic analysis of RNNs has often characterized their power by relating them to automata-theoretic classes of formal languages \citep[][\emph{inter alia}]{weiss-etal-2018-practical, peng-etal-2018-rational, merrill-2019-sequential}.
Recently, however, RNNs have largely been overtaken in NLP by a new class of models: transformers \citep{vaswani2017attention}.
Transformers are not autoregressive, and therefore less naturally resemble automata, posing challenges to characterizing their linguistic capacity or inductive biases in the same terms as RNNs. Instead, some recent work has related them to circuit complexity classes, a direction that we continue to pursue in this paper.
Drawing on classical circuit lower bound results, \citet{angluin2021} and \citet{hahn-2020-theoretical} derive theoretical limitations of transformers with \emph{hard} attention, meaning the attention distributions focus all probability 
mass on one index. Together, their results show that $\AC^0$---the class of languages recognizable by constant-depth circuit families---upper bounds the \new{formal languages hard-attention transformers can recognize.}

\begin{table}[t!]
    \centering
    \begin{tabular}{l|cc}
         & \new{Float ($\mathbb{F}$)} & \new{Rational ($\mathbb{Q}$)} \\ \hline
        \new{Hard ($\eta$)} & $\subseteq \AC^0$ & $\subseteq \AC^0$ \\
        \new{Saturated ($\zeta$)} & $\subseteq \TC^0$ & $= \mathsf{ALL}$
    \end{tabular}
    \caption{Summary of combined results from \citet{angluin2021} and this paper. Each cell $\alpha, \mathbb D$ characterizes the languages recognizable by transformers with attention function $\alpha$ and datatype $\mathbb D$ (floats $\mathbb F$ or rationals $\mathbb Q$). $\AC^0$ and $\TC^0$ are circuit complexity classes, and $\AC^0 \subset \TC^0$. $\mathsf{ALL}$ is the set of all formal languages over alphabet $\{0, 1\}$. See \autoref{sec:circuits} for formal definitions. Out of these results, we view saturated attention with floats as the best model of practical transformers.}
    \label{tab:my_label}
\end{table}

However, hard attention is a strong assumption, making it unclear how these results transfer to practical transformers. \new{For example, \citet{bhattamishra-etal-2020-ability} showed how transformers can solve synthetic counting tasks by using uniform attention patterns, which hard attention does not allow}. Motivated by this potential disconnect between theory and practice, we aim to extend circuit-based analysis to transformers with saturated attention: a generalization of hard attention that has been argued to approximate attention patterns acquired through gradient descent \citep{merrill2020parameter}. Broadly speaking, saturated attention goes beyond hard attention in that it can ``tie'' across a subset of positions, rather than selecting just \emph{one} position. The tied positions are then aggregated by averaging. Qualitatively, saturated attention heads can ``count'': a capability observed in transformers in practice \citep{bhattamishra-etal-2020-ability}. Further, \citet{merrill2020parameter} show that transformer training dynamics lead attention heads in several pretrained transformers to approximate saturated attention. In summary, saturated attention strictly generalizes hard attention and should more closely reflect the attention patterns acquired in practical transformers.

\new{Our main contributions are twofold. First, we show that saturated transformers can recognize languages outside $\AC^0$.
Then, as depicted in \autoref{tab:my_label}, we prove that transformers with floating point activations and saturated attention can only recognize formal languages in the circuit complexity class $\TC^0$, constituting an upper bound for a more realistic model of transformers than past results with hard attention.}

\section{Roadmap}

\new{In \autoref{sec:definitions}, we formally define our model of the transformer, including defining saturated attention in contrast to hard attention. \autoref{sec:circuits} introduces circuits in theoretical computer science and relevant complexity measures and classes for them.}

In \autoref{sec:unbounded}, we first briefly analyze saturated transformers with rational values \new{where the embedding, scoring, and activation functions are allowed to be any size-preserving function}. We find such transformers to be universally powerful. We also observe that when the positional embeddings are computed in time linear in the sequence length, saturated rational-valued transformers are exactly as powerful as the complexity class of their activation functions, because the full input sequence can be pooled to a single position, and an activation function can be used as an oracle over the full input sequence.
\new{However, this setup relies on the use of unrealistic embedding functions.} To move to a more realistic model of computation, we then focus on saturated transformers whose values are restricted to be \emph{floats}, which \new{have a coarser  granularity and, thus,} cannot encode the full input sequence into a single position.

Building on results of \citet{perez2018on}, we demonstrate in \autoref{sec:beyond-ac0} that saturated transformers with \new{\emph{float}} activations transcend the theoretical limitations of hard-attention transformers. In particular, we will show that they can recognize the majority language, which lies outside $\AC^0$.
We experimentally validate that transformers can learn to recognize the majority language. Taken together, these results suggest that the very weak characterization of hard-attention transformers does not hold \new{in practice for saturated or soft attention.}

In \autoref{sec:complexity-bounded}, we show that, on input sequences of length $n$, the size of each state vector in a transformer \new{over floats} is $\O(\log n)$ bits, similar to saturated LSTMs (cf.~\citealp{merrill-2019-sequential}).
Thus, the full transformer state at any given layer has size $\O(n \log n)$, although each feedforward block can only locally access a small, $\O(\log n)$ ``piece''.
Thus, while hierarchical representations (e.g., to process arbitrary-depth Dyck languages or reverse strings) can be implemented in a transformer, our result implies they must be \emph{distributed} in some way across $n$ state vectors, rather than represented compactly within a single vector.

Finally, in \autoref{sec:simulation}, we use the bounded size of transformer representations to upper bound that formal languages that can be recognized by a saturated transformers \new{with floating-point values}. In particular we show that they can be simulated by constant-depth threshold circuits, i.e., fall in $\TC^0$. Informally, this suggests that moving from hard attention to saturated attention can be thought of as extending the implicit class of circuit gates available in the network to include threshold gates.

\new{Our results make progress in the analysis of transformers by deriving upper bounds for a more realistic model of transformers than has previously been analyzed. RoBERTa, T5, and other pretrained transformers have been shown to be approximately saturated \citep{merrill2020parameter}, so our results imply $\TC^0$ may be a meaningful upper bound on the computation expressible within such networks. Our analysis also motivates future work further refining the characterization of saturated transformers, as well as comparing transformers with soft and saturated attention.}


\section{Definitions and Notation} \label{sec:definitions}

\new{
We will often use $w$ to refer to a string over any generic alphabet $\Sigma$, i.e., $w \in \Sigma^*$. Semantically, $w$ corresponds to the string a transformer receives as input. In contrast, we use $x$ and other symbols to refer to binary strings in $\{0, 1\}^*$. These binary strings will represent intermediate values within the transformer computation, rather than the raw input to the transformer.}

\subsection{Datatypes} \label{sec:datatypes}

Under our model, all values in the transformer are binary strings. In order to compute self attention and other operations over binary strings, we need to define datatypes describing the semantics of these binary strings as numbers. We will describe a semantics for binary strings as integers, as often comes up in circuit complexity. We then extend this to rational numbers and floats, which are necessary for representing the division operations that occur in attention heads within transformers.

\paragraph{\new{Unsigned} Integers}
We can interpret binary strings $x \in \{0, 1\}^*$ as \new{unsigned} integers in the standard way, \new{i.e., the numerical value of $x \in \{0, 1\}^n$ is}
\begin{equation*}
    \den{x}_\mathbb{Z} =  \sum_{i=1}^{n-1} 2^{i-1} x_i
\end{equation*}
We allow standard integer operations like $+_\mathbb{Z}, \cdot_\mathbb{Z}, <_\mathbb{Z}$. For example, ${101} +_\mathbb{Z} {1} = {110}$.

\paragraph{Rationals}
To interpret $r \in \{0, 1\}^*$ as a rational number, we first view it as \new{a sign bit $s$ along with a tuple of two unsigned integer substrings $\langle p, q \rangle$.}\footnote{Under the hood, we imagine the pair $\langle p, q\rangle$ is encoded by padding $p$ and $q$ to the same length with $0$'s and interweaving bits from each.} The numerical value represented by $r$ is
\begin{equation*}
    \den{r}_\mathbb{Q} = (2s - 1) \den{p}_\mathbb{Z} / \den{q}_\mathbb{Z} .    
\end{equation*}
Let $\mathrm{red}(p, q)$ return $\langle s, t \rangle$ where $s = p / \mathrm{gcd}(p, q)$ and $t = q / \mathrm{gcd}(p, q)$. Then, we can define arithmetic operations over two rationals $r = \langle p, q \rangle$ and $r' = \langle p', q' \rangle$ in the standard way:
\begin{align*}
    r +_\mathbb{Q} r' &= \mathrm{red}(p \cdot_\mathbb{Z} q' + p' \cdot_\mathbb{Z} q,\, q \cdot_\mathbb{Z} q') \\
    r \cdot_\mathbb{Q} r' &= \mathrm{red}(p \cdot_\mathbb{Z} p',\, q \cdot_\mathbb{Z} q') .
\end{align*}

\paragraph{Floats}
We define floats $\mathbb F$ as the subset of the rationals where the denominator is constrained to be a power of $2$.\footnote{More generally, the denominator may be taken to have a prime factorization of bounded length, although we work with the power of $2$ definition, which is both simpler and closely resembles conventional floating point datatypes.} Multiplication and addition are defined as for $\mathbb Q$, and are guaranteed to produce another float. Notably, division for floats is implemented by multiplying by an approximate multiplicative inverse, so it may be that $(x /_\mathbb{F} y) \cdot_\mathbb{Q} y \neq x$. See \autoref{sec:floats} for a more formal discussion.

In \autoref{sec:unbounded}, we will study transformers over rational values.
\new{From \autoref{sec:beyond-ac0} onwards}, we will then take the values in transformers to be floats unless otherwise stated.
Going forward, we will generally omit datatype subscripts from operations where they are clear from context.
\new{We will sometimes write $\mathbb D$ as a set in function signatures, e.g., $f : \mathbb D^k \to \mathbb D^k$. In this usage, it refers to the set $\{0, 1\}^*$, but it is often more intuitive to write the datatype shorthand (rather than $\{0, 1\}^*$) to hint at the intended semantics of the functional arguments.}

\new{
\paragraph{Size of Binary Strings} Under our model, integers, rationals, and floats are all abstractions built out of binary strings. For any $x \in \{0, 1\}^*$ (which can be interpreted semantically as an integer, float, or rational), we define its size $\abs{x}$ as the total length of $x$ measured in bits.
\new{We imagine a tuple $\langle p, q\rangle$ is encoded by padding $p,q$ to the same length with leading $0$'s, and interleaving bits from each sequence. This means the size of a rational is $2\max(\abs{p}, \abs{q}) + 1$.}
For example, the integer $2$ takes $2$ bits to specify, while the float $\frac{1}{2}$ takes $5$ bits ($1$ for the sign, $2$ for the numerator, $2$ for the denominator).
}

\new{
\paragraph{Size Preservation} 
We say that a function $f : \{0, 1\}^* \to \{0, 1\}^*$ is
size-preserving iff there exist constants $c, n$ such that for all inputs $x$ with $n \leq \abs{x}$, $\abs{f(x)} \leq c \cdot \abs{x}$.
Let $\mathcal P$ be the set of size-preserving functions.
While size-preserving functions are defined here over binary strings, they can be equivalently applied over integers, rationals, and floats, since these datatypes, as we have defined them, are just binary strings.
}

\subsection{Transformers} \label{sec:transformers}


\new{We define the following general transformer model, which can be parameterized to use different types of attention patterns and whose internal functions (e.g., feedforward blocks) can be computed by different function classes.}

\begin{definition}[Transformer]
A \emph{transformer} is a tuple  \new{$\langle \Sigma, \mathbb{D}, \alpha, L, H, \phi, \{s_{\ell,h}\}_{\ell,h=1}^{L,H}, \{f_\ell\}_{\ell=1}^L \rangle$} where
\begin{compactenum}
    \item $\Sigma$ is a finite input alphabet, i.e., the set of token types in a formal language.
    \item $\mathbb{D}$ is a scalar datatype, i.e., a semantics for interpreting binary strings as numbers. We will generally consider $\mathbb D = \mathbb F$.
    \item \new{$\alpha$ is an attention function that maps a vector of attention scores in $\mathbb{D}^n$ (for any $n$) to a normalized probability distribution, also in $\mathbb{D}^n$. In this paper we take $\alpha$ to be either hard ($\eta$) or saturated ($\zeta$) attention; see \S\ref{sec:attention-functions}.}
    \item \new{$L \in \mathbb N$ is the number of layers.}
    \item \new{$H \in \mathbb N$ is the number of heads.}
    \item \new{$\femb : \Sigma \times \mathbb{N} \to \mathbb D^m$ is a position-aware embedding function that maps a token and position to a vector, where $m$ is a multiple of $H$.}
    \item \new{For each $\ell,h$, the function $s_{\ell,h} : \mathbb D^m \times \mathbb D^m \to \mathbb D$ assigns attention scores to pairs of values.}
    \item \new{For each $\ell$, the function $\fact : \mathbb D^m \times \mathbb D^m \to \mathbb D^m$, maps a previous layer value and attention head output to a new value vector.}
\end{compactenum}
\end{definition}

On an input string $w \in \Sigma^n$, a transformer computes $L$ layers of output sequences $v_{\ell,1}, \cdots, v_{\ell,n}$ (for $\ell \leq L$), where each \new{$v_{\ell, i} \in \mathbb D^m$}.
In the $0$th layer, each token $w_i$ and its position $i$ are embedded into a value $v_{0,i}$.
Subsequent layers aggregate information from the previous value sequence $v_{\ell}$ using a \emph{multi-head attention mechanism}, and output a new value sequence $v_{\ell+1}$.
More formally, these layers are structured as follows:

\begin{compactenum}
    \item \new{\textbf{Embedding Layer:} $v_{0,i} = \femb(w_i, i)$.
    \item \textbf{Attention Head:} Each of the $H$ attention heads in layer $\ell$ maps the full previous sequence into a new value via $\fattn_{\ell,h}$ and then applies the attention function $\alpha$:
    \begin{align*}
        a_{\ell,h,i,j} &= \fattn_{\ell,h}(v_{\ell,i}, v_{\ell,j}) \\
        b_{\ell + 1,h,i} &= \sum_{j=1}^n \alpha(a_{\ell,h,i,:})_j \cdot v_{\ell,j} .
    \end{align*}}
    Crucially, the semantics for addition and multiplication here (as well as in the computation of $\alpha$) come from the datatype $\mathbb D$.
    \item \textbf{Activation Block:}\footnote{Let $V_{\ell,h}$ be a head's value matrix in the standard transformer parameterization. Then $\fact_\ell$ is computed by first multiplying each $b_{\ell,h,i}$ by $V_{\ell,h}$, aggregating the multiple attention heads, and applying the feedforward subnetwork.}
    \begin{equation*}
        \new{v_{\ell + 1, i} = \fact_{\ell + 1}(v_{\ell, i}, b_{\ell, :, i}) .}
    \end{equation*}
\end{compactenum}

\subsection{Attention Functions} \label{sec:attention-functions}

An attention function $\alpha$ maps a vector of scores \new{$a \in \mathbb{D}^n$} to a probability distribution over $1, \cdots, n$. Specifically, we consider two attention functions: \emph{hard} attention $\eta(a)$ and \emph{soft} attention $\zeta(a)$.

Hard attention collapses the attention scores to a one-hot distribution with all mass concentrated at one index. Let $\mathcal M(a) = \{ i \mid a_i = \max_j a_j \}$.
\begin{definition}[Hard attention] \label{def:hard}
Define hard attention $\eta(a)$ as
\begin{equation*}
    \eta(a)_j =
    \begin{cases}
        1 & \textrm{if} \; j = \min_{m \in \mathcal M(a)} m \\
        0 & \textrm{otherwise.}
    \end{cases}
\end{equation*}
\end{definition}
In contrast, saturated attention spreads probability mass evenly across ``tied'' scores.
\begin{definition}[Strong saturated attention; \citealt{merrill2020parameter}]
Define saturated attention $\zeta(a)$ as
\begin{equation*}
    \zeta(a)_j = \frac{1}{\abs{\mathcal M(a)}} \cdot
    \begin{cases}
        1 & \textrm{if} \; j \in \mathcal M(a) \\
        0 & \textrm{otherwise.}
    \end{cases}
\end{equation*}
\end{definition}
\citet{merrill-2019-sequential} shows how this form of attention can be derived by taking a large-norm limit of the network weights; a derivation can be found there.
Saturated attention reduces to hard attention when $\abs{\mathcal M(a)}=1$, and attends uniformly when $\abs{\mathcal M(a)} = n$. Both hard and uniform attention can be implemented with numerical stability, motivating \emph{weak} saturated \new{(or, ``uniform'')} attention:
\begin{definition}[Weak saturated attention]
Each head implements either hard attention (\autoref{def:hard}) \emph{or} the uniform pattern $\upsilon(a)_j = \frac{1}{n}$.
\end{definition}
In general, we will use ``saturated attention'' to refer to strong saturated attention and provide upper bounds for this setting. On the other hand, our lower bounds only use weak saturated attention, thereby showing that even weak saturated attention is more powerful than hard attention.

\subsection{Language Recognition} \label{sec:lang-rec}

\new{Finally, we define language recognition for transformers.}

\begin{definition}[Language recognition] \label{def:recognition}
Write $v_{\ell,i}(w)$ for the value of $v_{\ell,i}$ on input string $w$.
A transformer recognizes a language \new{$\mathcal L \subseteq \Sigma^*$} if there exists a \new{$\mathbb D$-valued} affine transformation $W, b$ such that, for all $w \in \Sigma^*$,
\begin{equation*}
    W \cdot v_{L,1}(w) + b > 0 \iff w \in \mathcal L .
\end{equation*}
\end{definition}
This says the decision problem of recognizing $\mathcal L$ must be linearly separable using the first value in the last layer of the transformer. In practice, the first token in a transformer is often set to \texttt{CLS}, and its output can be passed to a classifier during finetuning \citep{devlin-etal-2019-bert}. This inspires \autoref{def:recognition}.
There are other potential ways to define language recognition and generation for transformers \citep{hewitt-etal-2020-rnns, yao2021selfattention}, but they do not lead to meaningful differences for our purposes.

\new{Finally, we define $\trans(\mathbb D)$ as the set of languages recognizable by some saturated transformer over $\mathbb D$, where the internal functions can be any size-preserving function.\footnote{\new{The name $\trans$ standards for ``averaging hard attention transformer'', and is taken from \citet{angluin2021}.}}}
\begin{definition} \label{def:trans-class}
\new{
Let $\trans(\mathbb D)$ be the set of languages $\mathcal L$ such that there exists a transformer $\langle \Sigma, \mathbb D, \zeta, L, H, \phi, s_{\ell,h}, f_\ell \rangle$ that recognizes $\mathcal L$ where each $\phi, s_{\ell,h}, f_\ell \in \mathcal P$.\footnote{\new{To apply size preservation to the embedding function $\phi$, we consider the size of a token to be $\log(\abs{\Sigma})$.}}
}
\end{definition}
\new{We note that size preservation is a weak condition to assume about the internal functions in practical transformers: since any linear-time-computable
function is size-preserving, it is strictly weaker than assuming the internal functions can be computed in linear time. To further justify this condition, we explicitly show in \autoref{sec:justifying-l} that the component functions within transformers are size-preserving.}

\new{\section{Circuit Complexity} \label{sec:circuits}} 

Circuit complexity is a branch of computational complexity theory that studies circuit families as a model of computation.\footnote{\new{For more reference material on circuit complexity, \new{we refer the reader to chapters 6 and 14 of \citet{arora2009computational} or chapters 1 and 2 of the \emph{Handbook of Theoretical Computer Science}, Volume A \citep{chapter1, chapter2}.}}}
\new{Intuitively, circuits are useful for formally studying the types of computational problems that can be efficiently solved with parallelism, as the depth of a circuit corresponds to the runtime of a program on an idealized, fully parallel computer.}
We review background on circuits, circuit families, and relevant complexity measures and classes.

\paragraph{Circuits} For a fixed $n$, a \emph{circuit} is a computation graph, where leaves correspond to input bits $x_i$ and their negations $\neg x_i$, and the internal nodes are logic gates (typically $\wedge$ and $\vee$), with one labeled as the output node.
The gates can conventionally be taken to have either binary or unbounded fan-in.
The circuit computes a function $f : \{0, 1\}^n \to \{0, 1\}$ by substituting the input values into the leaf nodes, propagating the computation through the graph, and returning the value of the output node.
\autoref{fig:example-circuit} shows an example circuit that takes inputs of length $5$, and returns whether they contain the bigram $11$.

\begin{figure}[t!]
    \centering
    \includegraphics[width=0.8\columnwidth]{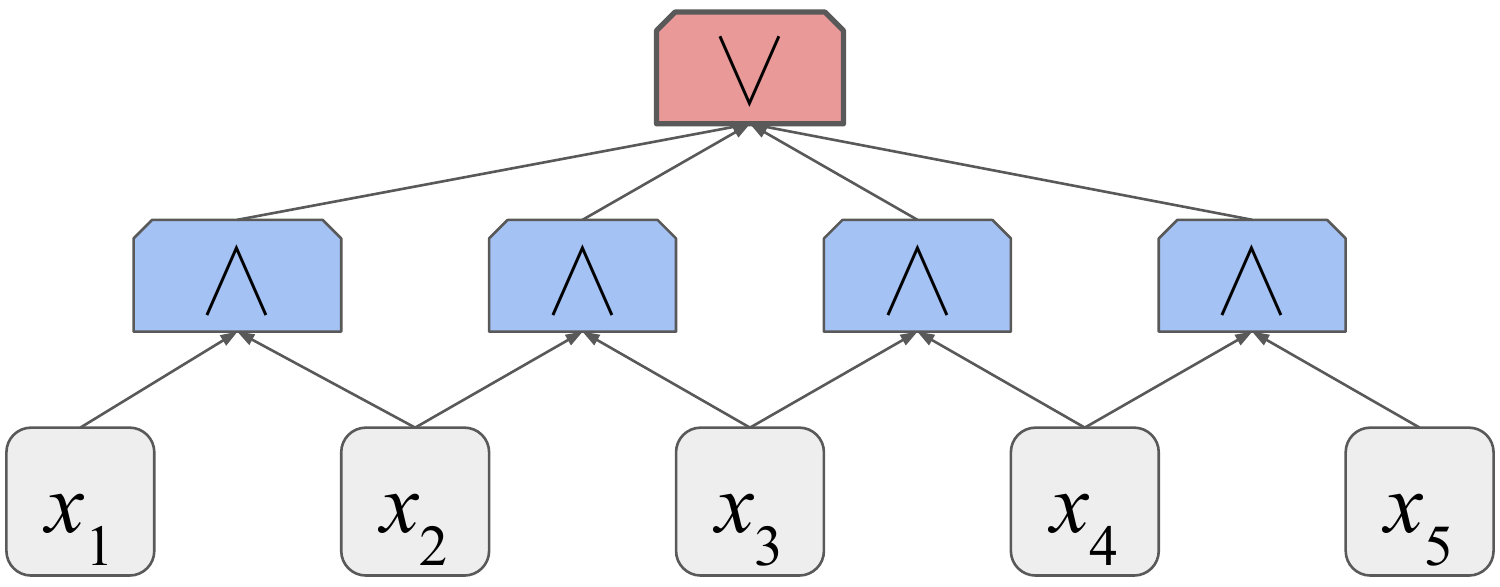}
    \caption{A circuit that takes a string $\in \{0,1\}^5$ and returns whether it contains the bigram $11$.}
    \label{fig:example-circuit}
\end{figure}

\paragraph{Circuit Families} A \emph{circuit family} is an ordered set of circuits $\{ C_n \}_{n \in \mathbb{N}}$ where each circuit is identified with a particular input size $n$. We say a circuit family recognizes a formal language $\mathcal L \subseteq \{0, 1\}^*$ iff, for all $w \in \mathcal L$,\footnote{Similarly, for any alphabet $\Sigma$ and $\mathcal L \subseteq \Sigma^*$, we interpret $w_i$ as a one-hot vector over $\Sigma$ and define the family to recognize $\mathcal L$ iff, for all $w \in \mathcal L$, $C_{\abs{w} \cdot \abs{\Sigma}}(w) = 1 \iff w \in \mathcal L$.}
\begin{equation*}
    C_{\abs{w}}(w) = 1 \iff w \in \mathcal L .
\end{equation*}


\paragraph{Circuit Complexity}
Two important notions of complexity for a circuit are its size and depth. The size of a circuit is the number of gates. The depth is the longest path from an input node to the output node. For a circuit family, both quantities can be expressed as functions of the input size $n$. A \emph{circuit complexity class} is a set of formal languages that can be recognized by circuit families of a certain size, depth, and set of gates. In particular, we will discuss the classes $\AC^0$ and $\TC^0$.
\begin{definition}
$\AC^0$ is the set of languages $\mathcal L \subseteq \{0, 1\}^*$ such that there exists a circuit family recognizing $\mathcal L$ with unbounded arity $\{\wedge, \vee\}$ gates, $\poly(n)$ size, and $\O(1)$ depth.
\end{definition}
\new{Intuitively, $\AC^0$ represents the class of problems that are highly parallelizable when the computational primitives are standard logic gates. In contrast, $\TC^0$ will also represent highly parallelizable computation, but when the gates are expanded to include \emph{threshold gates}.}
For a bitstring $x \in \{0,1\}^*$, define the threshold gate $\theta_{{\geq}k}(x)$ to return $1$ iff $\geq k$ bits in $x$ are $1$, and equivalently for $\leq k$.
For example, $\theta_{{\geq}3}(110011) = 1$.
\begin{definition}
$\TC^0$ is the set of languages $\mathcal L \subseteq \{0, 1\}^*$ such that there exists a circuit family recognizing $\mathcal L$ with unbounded arity $\{\wedge, \vee, \theta\}$ gates, $\poly(n)$ size, and $\O(1)$ depth.
\end{definition}

It is known that $\AC^0 \subset \TC^0 \subseteq \NC^1$, where $\NC^1$ denotes the languages recognizable by $\O(\log n)$-depth circuits with bounded gate arity. \new{Whether or not the latter containment between $\TC^0$ and $\NC^1$ is strict is an open question.}
Whereas parity and other basic regular languages are outside $\AC^0$ \citep{furst81parity}, $\TC^0$ properly contains parity, although it is unknown whether it contains \emph{all} the regular languages. Between $\AC^0$ and $\TC^0$ lies the class $\mathsf{ACC}^0$ \citep{yao1990acc}.

\new{
\paragraph{Uniformity} The circuit classes defined above (and which we will use in this paper) are \emph{non-uniform}, meaning circuits for different input sizes are not constrained to have any relation to each other. Non-uniform circuit families can recognize some uncomputable languages, such as the language of strings $1^k$ such that Turing machine $k$ does not halt on the null input \citep[cf.][]{arora2009computational}. In contrast, the \emph{uniform} variants of circuit families are constrained such that a log-space Turing machine must output a string encoding of circuit $C_n$ on the input string $1^n$, forcing any language the circuit family can recognize to be computable. For these uniform classes (which we write with a \emph{u} prefix), it is known that
\begin{equation*}
    \u \TC^0 \subseteq \u \NC^1 \subseteq \mathsf L \subseteq \mathsf P ,
\end{equation*}
where $\mathsf L$ and $\mathsf P$ denote the conventional complexity classes of log-space and polynomial-time decision problems. Thus, it is unknown whether $\u \TC^0$ is restricted compared to general polynomial-time computation, but if we accept the common conjecture that one (if not all) of the above containments are strict, then $\u \TC^0$ forms a restricted family of problems compared to $\mathsf P$, which, intuitively, are more parallelizable than other problems in $\mathsf P$.
}

\section{Aren't Transformers Universal?} \label{sec:unbounded}
\new{We now begin our analysis of saturated transformers.}
\citet{angluin2021} and \citet{hahn-2020-theoretical} were able to give upper bounds on the power of hard attention without imposing any constraints on the embedding, \new{scoring}, and activation functions. The same will not be the case \new{with saturated} attention:
any bounds on transformers will require leveraging \new{some properties constraining their internal functions}.
\new{One property we use} will be size preservation.
We will first show though that size preservation is not enough on its own: deriving a nontrivial upper bound will depend on subtle assumptions about the transformer's datatype.

\new{With rational values and size-preserving internal functions}, \new{we will show} saturated transformers can recognize \emph{any} formal language, i.e., the class $\mathsf{ALL} = \{ \mathcal L \mid \mathcal L \subseteq \{0, 1\}^* \}$.
Our construction resembles the universal approximation construction of \citet{Yun2020Are}, which relies on the ability of the transformer to uniquely encode the full input string into a single value vector.
After the full sequence is encoded locally into a single vector, the activation block can be used as a black box to recognize any language.
\begin{restatable}{proposition}{rationals} \label{thm:rationals}
    $\trans(\mathbb Q) = \mathsf{ALL}$.
\end{restatable}
\begin{proof}
We construct a $1$-layer rational-valued transformer with a single head to recognize every string $w$ in any formal language $\mathcal L \in \mathsf{ALL}$. We will omit $\ell,h$ subscripts. Let $p_i$ denote the $i$th prime number. The embedding layer encodes each input token according to
\begin{equation*}
    \femb(w_i, i) =
    \begin{cases}
        1 / p_i & \textrm{if} \; w_i = 1 \\
        0 & \textrm{otherwise.}
    \end{cases}
\end{equation*}
\new{Since $p_i \sim i \log i$ for large $i$ by the prime number theorem \citep[cf.][]{goldstein1973history}, the number of bits needed to represent $\femb(w_i, i)$ is
\begin{align*}
    \leq c \log( i \log i) \leq c \log(i^2) = 2c \log i.
\end{align*}
Since $i$ had size $\log i$, this implies $\femb$ is size-preserving.}

Now, we define a single uniform attention head that sums across all $i$, outputting $\sum_{w_i=1} \frac{1}{p_i}$. The denominator $q$ of this sum is the product $\prod_{i=1} p_i$. 
Observe that $w_i = 1$ if and only if $p_i$ divides $q$.
Thus, we can define a function $f$ that extracts the input sequence $w$ from $q$ by checking whether, for each $i$, $p_i$ divides $q$. We let $g$ be a function recognizing $L$, and set $\fact = g \circ f$. \new{The output of the transformer will now compute whether $w \in \mathcal L$, since $f$ outputs an encoding of the original input sequence $w$, and $g$ decides whether $w \in \mathcal L$.}
Note that any function solving a decision problem is size-preserving, hence $\fact \in \mathcal P$.
\end{proof}

\autoref{thm:rationals} says that our transformer architecture parameterized with a rational datatype can recognize any formal language.
But a construction of this form feels unrealistic for two reasons. First, it requires the embedding layer to implement an unconventional prime encoding scheme in the embedding layer.
Second, we are using the activation layer as a black box to recognize any language---even uncomputable ones! On the other hand, the feedforward subnetworks used in practice in transformers cannot even implement all computable functions when the weights are fixed independent of the sequence length $n$.
We can get around both these issues by instead restricting the datatype to floats, which is the direction we will pursue in the remaining sections.\footnote{It may also be possible to derive tighter bounds for rational-valued transformers by imposing stronger constraints on \new{the internal functions}. However, with floats, we will see that size preservation is sufficient to derive a tighter characterization of transformers' power. We leave this alternate direction to future work.}

\subsection{Resource-Bounded Transformers}
\label{subsec:resource-bounded}

In \autoref{sec:resource-bounded}, we develop an alternate perspective on the universality of transformers, \new{showing that, if the embedding function is allowed to be computed in time linear in the sequence length, then the transformer's complexity is equivalent to its activation functions' complexity.}
\begin{proposition}[Informal] \label{thm:resource-bounded-informal}
 \new{If $\phi$ can be any function computable in time linear in $n$, and the scoring and activation functions can be computed in $T(m)$ time on inputs of size $m$ with $T(m) \geqslant m$, then languages recognizable by the transformer are $\mathsf{TIME}(T(m))$.}
\end{proposition}
\autoref{sec:resource-bounded} contains a formal statement and proof.
For example, \new{allowing polynomial-time functions inside the transformer} implies that the transformer will recognize exactly \new{the complexity class $\mathsf{P}$}.
\new{A major unrealism about this setup is the assumption that $\phi$ can be an arbitrary function computable in time linear in $n$, motivating our main results in a more constrained setting in \autoref{sec:simulation}.}

\new{
\subsection{Discussion}
We are not stating the results in this section as evidence that practical transformers are capable of universal or arbitrary polynomial computation. Rather, the unnaturalness of these constructions (specifically, the prime numbers based position encoding) motivates us to slightly constrain our model of the transformer in a realistic way: we will switch the datatype from rationals to floats, since even using only simple uniform attention, a model with rationals and unconstrained internal functions is universal. We will soon see that this realistic constraint prevents universal simulation, and in fact bounds the capacity of the saturated transformer within $\TC^0$.}



\section{Beyond Hard Attention\new{, with Floats}} \label{sec:beyond-ac0}


We now move to the setting of saturated transformers over floats.
\citet{angluin2021} identified that hard-attention transformers can only recognize languages within $\AC^0$.
In contrast, saturated transformers over floats can recognize the ``majority'' language $\maj$, which is known to lie outside $\AC^0$ \citep{furst81parity}.
\citet[][Prop.~3.3]{perez2018on} show how $\maj$ can be recognized by transformers. In \autoref{thm:majority}, we offer a simpler construction that leverages \new{only a single uniform attention head}, as opposed to the model of transformers they were considering.
\new{Thus, this construction is achievable with saturated attention.}

\begin{figure}
\centering
\begin{lstlisting}[language=RASP]
frac0 = aggregate(
    select_all,
    indicator(tokens == 0));
frac1 = aggregate(
    select_all,
    indicator(tokens == 1));
maj = frac1 > frac0;
\end{lstlisting}
\caption{A program recognizing $\maj$ in RASP, a programming language designed to abstract away details of transformer computation \citep{weiss2021thinking}. \texttt{frac\{0,1\}} measure the fraction of inputs that are $0$ or $1$. Then \texttt{maj} checks whether \texttt{frac1 > frac0}.}
\label{fig:majority-code}
\end{figure}

\begin{proposition} \label{thm:majority}
    $\trans(\mathbb F) \not\subseteq \AC^0$.
\end{proposition}
\begin{proof}
Let $\#_\sigma(w) \in \mathbb{N}$ denote the number of $\sigma$ tokens in string $w \in \{0,1\}^*$. Let $\#(w)$ denote a count vector where each element corresponds to some $\sigma \in \{0, 1\}$. We define $\maj$ as follows:
\begin{equation*}
    \maj = \big\{ w \in \{0, 1\}^+ \mid \#_1(w) > \#_0(w) \big\} .
\end{equation*}
We will construct a $1$-layer transformer with a single head to recognize $\maj$, omitting $\ell,h$ subscripts from $\fattn, \fact, x, b$. \autoref{fig:majority-code} gives the same construction in RASP \citep{weiss2021thinking}.

Let $x_i = \femb(w_i, i)$ be a $1$-hot encoding of $w_i$. For all $i,j$, set $\fattn(x_i, x_j) = 1$, resulting in a single head attending everywhere:
\begin{equation*}
    b_i = \frac{\#(w)}{n} .
\end{equation*}
Finally, set $\fact(b_i)$ to return whether $\#_1(w) / n > \#_0(w) / n$, which, for $n > 0$, is true iff $w \in \maj$.
\end{proof}
Notably, the construction in \autoref{thm:majority} is not just possible within our generalized transformer framework, but can also be implemented by the standard parameterization of $\femb, \fattn$, and $\fact$ in real transformers \citep{vaswani2017attention}. The uniform attention pattern can be implemented by setting all query and key attention parameters to $0$. Then, we can use the affine transformation that aggregates the head outputs to compute the tuple:
\begin{equation*}
    \left \langle \frac{\#_1(w) - \#_0(w)}{n}, 0 \right \rangle .
\end{equation*}
This tuple is then passed through layer normalization \citep{ba2016layer}, resulting in a new tuple $\langle t_1, t_2 \rangle$. Crucially, $t_1 > t_2$ if and only if the same applies to the quantities in the original tuple. Thus, a linear classifier can decide whether $t_1 > t_2$ to successfully recognize the language, as per \autoref{def:recognition}.

\begin{figure}
    \centering
    \includegraphics[width=\columnwidth]{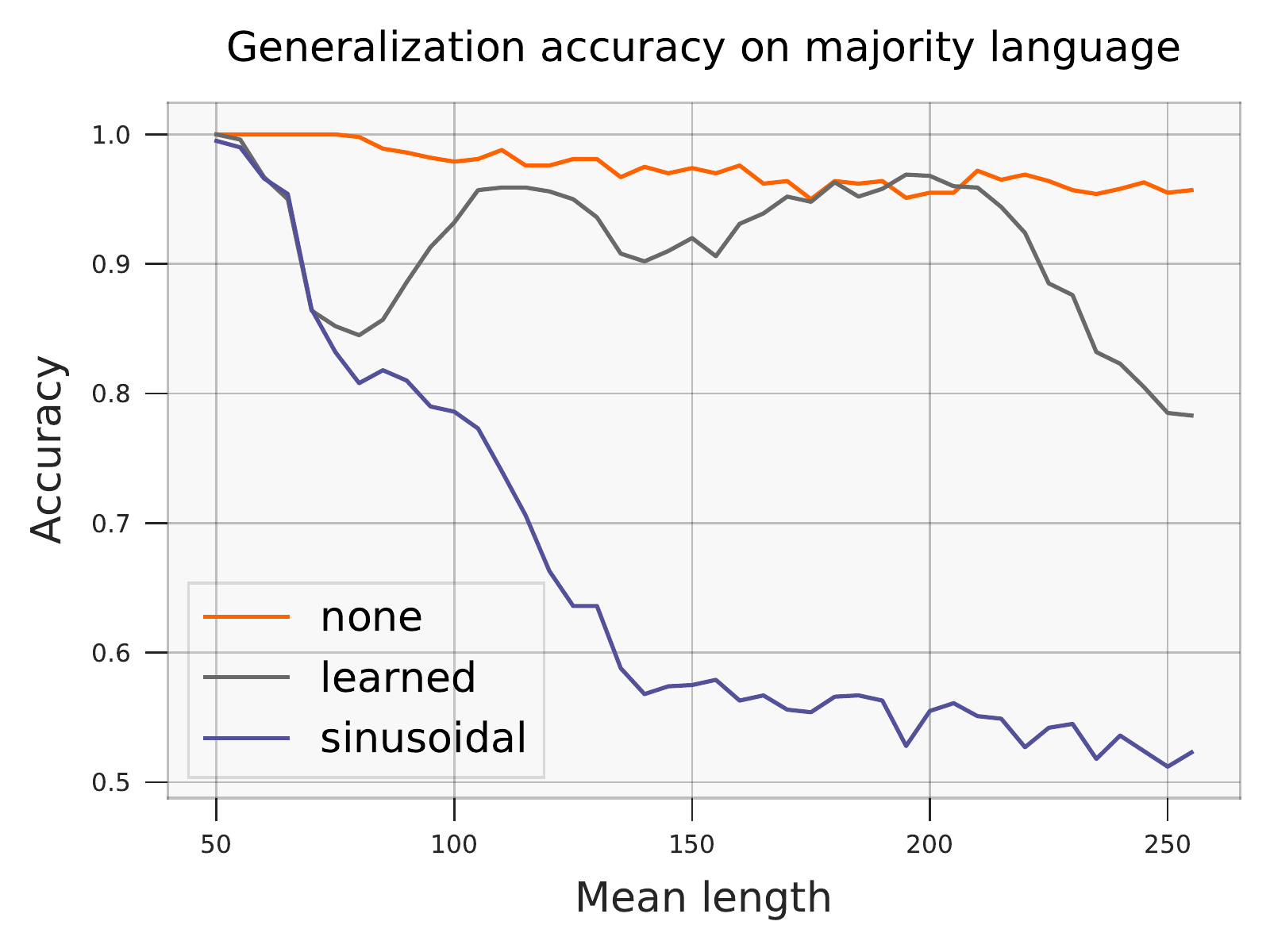}
    \caption{In practice, transformers can learn the majority language (which lies outside $\AC^0$). We train $1$-layer transformers on majority, where each line represents a different positional encoding scheme. Training string length was binomial with $n=100$. Trained models were then evaluated on generalization sets with $n$ ranging from $100$ to $500$.  Mean length ($x$ axis) is $n/2$.}
    \label{fig:accs}
\end{figure}

\subsection{Empirical Validation}

In \autoref{fig:accs}, we show empirically that a $1$-layer transformer can learn and generalize $\maj$.
This supports our argument that the theoretical limitations of hard-attention transformers do not apply to practical transformers.
We train with three different types of positional encoding: none, meaning no positional information; learned, where each position gets a trainable embedding vector, and the sinusoidal scheme of \citet{vaswani2017attention}.
The model with no positional embeddings generalizes the best, followed by the learned embeddings.
\new{It appears that while $\maj$ is in the capacity of the transformer, the standard sinusoidal positional embedding scheme provides the wrong inductive bias for learning it.}
This recalls the finding of \citet{yao2021selfattention} that the choice of positional encodings seems to greatly impact the transformer's ability to generalize formal language tasks to longer sequences.

\section{\new{Size of Transformer Values}} \label{sec:complexity-bounded}

\new{The theoretical limits on hard-attention transformers were derived by \citet{angluin2021} by bounding the size in bits of the representation $v_{\ell,i}$ at each layer $\ell$ and position $i$. Specifically, they show that the value $v_{\ell,i}$ is representable in $\O(\log n)$ bits on input sequences of length $n$.
Thus, each value can only contain limited information about the input sequence, intuitively explaining their upper bound on hard-attention transformers.
Inspired by their analysis, this section will show that, in a saturated transformer, each $v_{\ell,i}$ also has a size of $\O(\log n)$ bits.
Later in \autoref{sec:simulation}, we will use this property to show that saturated attention transformers are limited in the formal languages they can recognize.}

\subsection{\new{Size of Float Sums}}  \label{sec:redundant}

\new{How many bits does it take to represent the value of an attention head within a saturated transformer?}
As a naive bound, the \new{output of} a saturated attention head is specified by \new{a float for each of $n$ values attended over from the last layer, which would take at least linearly many bits in $n$}. \new{However, this upper bound on its size is not tight.} Instead, we will show that all head and activation values can be represented in $\O(\log n)$ bits.
Our analysis will rely heavily on the following lemma:
\begin{lemma} \label{lem:lin-bits}
\new{
Let $v_1, \cdots, v_n$ be a sequence of floats, each with size at most $z$.
Then there exists $c$ such that $\sum_{i=1}^n v_i$ has size at most $4cz + 2\log n + 1$.
}
\end{lemma}
\begin{proof}
\new{
Let $p_i, q_i$ denote the numerator and denominator of the floating point $v_i$, respectively.}
Similarly, let $p_s, q_s$ be the numerator and denominator of the float $s$.
By assumption, there exists $c$ such that each $p_i, q_i$ both have size $\leq cz$ for large enough $n$.
We let $p_\mathrm{max} = \max_i p_i$ and analogously for $q_\mathrm{max}$. Since all $q_i$'s are powers of $2$, the numerator $p_s$ is
\begin{equation*}
    \sum_{i=1}^n p_i \cdot \frac{q_\mathrm{max}}{q_i} \leq n p_\mathrm{max} q_\mathrm{max} ,
\end{equation*}
which, represented as an integer, has size:
\begin{equation*}
\new{
    \leq \log n + cz + cz = 2cz + \log n .
}
\end{equation*}
On the other hand, the denominator $q_s = q_\mathrm{max}$, which has size $\leq z$. \new{Therefore, the float representing the sum has size
\begin{align*}
    &\leq 1 + 2 \max(2cz + \log n, z) \\
    &= 4cz + 2 \log n + 1 ,
\end{align*}
which completes the proof.
}
\end{proof}

In particular, we will use \autoref{lem:lin-bits} to show that, \new{when each of a sequence of $n$ values has size $\O(\log n)$, the sum will also have size $\O(\log n)$.}

\subsection{\new{Size of Transformer Values}}

\new{We will now leverage \autoref{lem:lin-bits} to show that the values are of bounded size in any transformer over floats with an elementwise-size-preserving attention function.
\begin{definition}
A function $\alpha : \mathbb D^n \to \mathbb D^n$ is elementwise-size-preserving if, for $1 \leq i \leq n$, the function $x_i \mapsto \alpha(x)_i$ is size-preserving (where $x \in \mathbb{D}^n$).
\end{definition}
Note that saturated attention satisfies this definition. We are ready to prove a theorem bounding the size of the representations in transformers with elementwise-size-preserving attention.}
\begin{proposition} \label{thm:logspace}
\new{For any transformer over $\mathbb{F}$ with $\phi, s_{\ell,h}, f_\ell \in \mathcal P$ and $\alpha$ elementwise-size-preserving, for all $\ell \leq L$ and $i \leq n$, $v_{\ell,i}$ has size $\O(\log n)$.}
\end{proposition}
\begin{proof}
By induction over $\ell$. \new{The proof follows the definition of transformer computation in \autoref{sec:transformers}.}
\new{\paragraph{Base Case} $w_i \in \Sigma$ has size $\O(1)$, and $i \in [n]$ has size $\O(\log n)$.
Since $\phi \in \mathcal P$, $v_{0,i} = \phi(w_i, i)$ has size $\O(\log n)$ for all $i$.}

\paragraph{Inductive Case}
\new{
Assume $v_{\ell,i}$ has size $\O(\log n)$.
Since $s_{\ell+1,h} \in \mathcal P$, $a_{\ell+1,h,i,j} = s_{\ell+1,h}(v_{\ell,i}, v_{\ell,j})$ has size $\O(\log n)$ for all $i,j$.
Since $\alpha$ is elementwise-size-preserving, we can conclude that $\alpha(a_{\ell+1,h,i,:})_j$ also has size $\O(\log n)$ for all $h,i,j$. Multiplying two floats is size-preserving (cf. \autoref{sec:justifying-l}), so $\alpha(a_{\ell+1,h,i,:})_j \cdot v_{\ell, j}$ has size $\O(\log n)$ for all $h, i,j$.
We then apply \autoref{lem:lin-bits} to conclude that $b_{\ell+1,h,i}$ has size $\O(\log n)$, where, recall,
\begin{equation*}
    b_{\ell + 1,h,i} = \sum_{j=1}^n \alpha(a_{\ell,h,i,:})_j \cdot v_{\ell,j} .
\end{equation*}
Finally, computing $v_{\ell+1,i} = f_{\ell+1}(v_{\ell,i}, b_{\ell,:,i})$, we conclude that $v_{\ell+1,i}$ has size $\O(\log n)$ for all $i$ by size preservation.
}
\end{proof}

\begin{corollary} \label{cor:saturated-size}
\new{For any saturated transformer over $\mathbb F$ with size-preserving internal functions, for all $\ell \leq L$ and $i \leq n$, $v_{\ell,i}$ has size $\O(\log n)$.}
\end{corollary}

\new{\autoref{cor:saturated-size} follows because saturated attention is elementwise-size-preserving. Softmax attention, on the other hand, is not guaranteed to fulfill this property, since it requires computing the exponential function. This technical challenge prevents generalizing our technique to soft attention.}

\subsection{Discussion}
Similar to hard-attention transformers \citep{angluin2021}, the size of each vector representation in a saturated transformer over floats is $\O(\log n)$.
\new{This is enough memory for individual vectors to ``count'', a behavior that has been observed in both LSTMs \citep{weiss-etal-2018-practical} and transformers \citep{bhattamishra-etal-2020-ability}.}
\new{On the other hand, $\O(\log n)$ space is not enough memory for individual vectors} (for example, the \texttt{CLS} output) to encode arbitrarily large combinatorial objects like trees.
However, transformers are \emph{not} limited to computing in an ``online'' fashion where tokens are consumed sequentially, meaning their effective state is $n$ values of size $\O(\log n)$.
Notably, trees with $n$ leaves can be encoded in a distributed fashion across $n$ values of size $\O(\log n)$. One construction for this is, at index $i$, to store $w_i$ and $i$, along with a pointer $j$ to the parent. Since $i,j$ can both be represented in $\log n$ bits, each vector uses only $\O(\log n)$ bits.

Additionally, the $\O(\log n)$ space bound has implications from the perspective of circuit complexity. While saturated attention cannot be simulated in $\AC^0$, we will show in \autoref{sec:simulation} that saturated transformers \new{over $\mathbb F$} \emph{can} be simulated by $\TC^0$ circuits.

\section{Threshold Circuit Simulation} \label{sec:simulation}

We have proved that each value vector in a saturated transformer over floats has $\O(\log n)$ size. Now, we show how this implies saturated transformers can be simulated by $\TC^0$ circuits. \new{Our results heavily leverage the following lemmas:}


\new{
\begin{lemma}[\citealt{angluin2021}] \label{lem:angluin}
Any function $f : \{0, 1\}^c \to \{0, 1\}^d$ can be computed by a boolean circuit of depth $3$ and size at most $(2^c + c + 1) d$.
\end{lemma}
\noindent So that our results are self-contained, we reproduce a proof of this lemma in \autoref{sec:angluin-proofs}. Applying \autoref{lem:angluin} to a size-preserving function with at most $c \log n$ input bits immediately yields:
\begin{lemcorollary} \label{cor:angluin}
Any size-preserving function with at most $c \log n$ input bits can be computed by a boolean circuit of depth $3$ and polynomial size.
\end{lemcorollary}
In other words, such functions can be computed with $\AC^0$ circuits. In addition, we will show that the sum of $n$ floats of size at most $c \log n$ can be computed by $\TC^0$ circuits.

\begin{lemma} \label{lem:sum-tc0}
Let $v_1, \cdots, v_n$ be a sequence of floats, each with size at most $c \log n$ for some $c$. Then the sum $\sum_{i=1}^n v_i$ is computable by a threshold circuit of constant depth and polynomial size.
\end{lemma}
\begin{proof}
Let the integers $p_i, q_i$ be the numerator and denominator of $v_i$. We first compute $q_\mathrm{max}$, the maximum $q_i$, using an $\AC^0$ circuit that compares all pairs $q_i, q_j$, and returns the first $q_i$ such that $q_j \geq q_j$ for all $j$. We then use the fact that multiplication and right shift ($q_i$ is a power of $2$) are in $\TC^0$, in order to compute
\begin{equation*}
    r_i = p_i \frac{q_\mathrm{max}}{q_i}
\end{equation*}
in parallel for all $i$. Note that $q_\mathrm{max}$ and $q_i$ are both powers of $2$, so the division will be exact. Next, we leverage the fact that the sum of $n$ integers of size $\O(\log n)$ is in $\TC^0$ \citep{tc0notes}, in order to compute the numerator of the sum $p' = \sum_i r_i$. We select the denominator as $q' = q_\mathrm{max}$. Finally, we can add an $\AC^0$ circuit that ``reduces'' the fraction by removing shared trailing zeros from $p', q'$, which is possible by \autoref{cor:angluin}. Thus, we have constructed a $\TC^0$ circuit to compute the sum of $n$ floats with size $\O(\log n)$.
\end{proof}
}

\new{We now construct a $\TC^0$ circuit that simulates a saturated transformer over floats.}
\begin{proposition}
$\trans(\mathbb F) \subseteq \TC^0$.
\end{proposition}
\begin{proof}
For each $n$, we construct a $\TC^0$ circuit that simulates a saturated transformer on inputs of size $n$. We construct the circuit modularly, with one subcircuit for the attention mechanism, and another for the feedforward subnetwork.

\paragraph{Attention Head}
\new{Fix a single head in some layer.} We will construct a $\TC^0$ subcircuit that simulates the attention mechanism \new{at position $i$}.
The head attends over vectors $v_1, \cdots, v_n$. \new{For all $j$,} $v_j$ has size $\O(\log n)$ by \autoref{thm:logspace}.
In parallel for each $j$, we compute the scores $a_{i,j} = \fattn(v_i, v_j)$ with an $\AC^0$ circuit by \autoref{cor:angluin}.
\new{We then compute $a_{i,\mathrm{max}} \triangleq \max_j a_{i,j}$ with an $\AC^0$ circuit by comparing all $v_j$ pairwise, and selecting the first $v_k$ such that $v_k \geq v_j$ for all $j$.
We then compute ``masked'' values $u_{i,j}$ for each $j$ via an $\AC^0$ circuit by \autoref{lem:angluin}:}
\begin{equation*}
    u_{i,j} \triangleq \begin{cases}
        v_j & \textrm{if} \; a_{i,j} \geq a_{i,\mathrm{max}} \\
        0 & \textrm{otherwise.}
    \end{cases}
\end{equation*}
\new{We then compute the sum $s_i \triangleq \sum_{j=1}^n u_{i,j}$ by \autoref{lem:sum-tc0}. By \autoref{lem:lin-bits}, $s_i$ has size $\O(\log n)$. Now, we similarly define}
\begin{equation*}
    z_{i,j} \triangleq \begin{cases}
        1 & \textrm{if} \; a_{i,j} \geq a_{i,\mathrm{max}} \\
        0 & \textrm{otherwise.}
    \end{cases}
\end{equation*}
\new{Using an analogous sum construction with $z_{i,j}$ instead of $u_{i,j}$, we can use a $\TC^0$ circuit to compute $\abs{\mathcal M(a)}$: the number of $j$ such that $a_{i,j} \geq a_{i,\mathrm{max}}$.
Finally, since dividing floats is in $\TC^0$ (cf.~\autoref{sec:floats}), we can compute the head output as $s_i / \abs{\mathcal M(a)}$, which has size $\O(\log n)$ by size preservation of division.}


\paragraph{Feedforward}
\new{As input, $\fact$ receives $v_i$ as well as $H$ head outputs, all of which have size $\O(\log n)$.}
As the total size of the input is $\O(\log n)$, we can use \autoref{cor:angluin} to compute the output of $\fact$ with an $\AC^0$ circuit.
\new{The size of the output is $\O(\log n)$ by size preservation of $f$.}
\new{The same idea holds for $\femb$ as well as the linear classification head.}


\new{We have simulated each transformer component with a $\TC^0$ subcircuit, completing the proof.}
\end{proof}

\subsection{Discussion}

Recall that, \new{over rationals}, we found that \new{size-preserving saturated} transformers could recognize any language.
In contrast, we have now shown that using floating-point representations \new{places such transformers} within $\TC^0$.
In this paper, we have only considered non-uniform $\AC^0$ and $\TC^0$, as opposed to the uniform variants of these classes, \new{which are more closely connected} to familiar formal language classes like the regular \new{and context-free} languages \citep[cf.][]{cojocaru2016advanced, mahajan2007polynomial}. As transformers satisfy some intuitive notion of uniformity, an open question is \new{whether saturated transformers also fall into uniform $\TC^0$}.

\section{Conclusion}
Compared to hard attention, saturated attention adds theoretical power to transformers.
We \new{showed} that saturated attention lets transformers recognize languages outside $\AC^0$, which is the upper bound \new{with hard attention}. Further,
while \new{saturated} transformers with rational values and \new{size-preserving internal functions} can recognize any language,
we characterize the limits of \new{size-preserving saturated transformers} with \emph{floats}.
Specifically, saturated transformers with float values fall in $\TC^0$, a more powerful circuit class than $\AC^0$. Thus, going from hard to \new{saturated} attention can be understood as augmenting the model with threshold gates. This illustrates one way that the circuit complexity paradigm characterizes the power of transformers. \new{Going forward, there are many interesting open questions that circuit analysis can answer, such as comparing the power of saturated and soft attention, and refining existing upper bounds for transformers in terms of uniform circuit families.}

\iftaclpubformat
\section*{Acknowledgments}
Thanks to Yiding Hao, Dana Angluin, and Robert Frank for sharing an early draft of their work. We also appreciate helpful feedback from Dana Angluin, Matt Gardner, Yoav Goldberg, Michael Hahn, Kyle Richardson, and Roy Schwartz.
\fi

\input{appendix}

\bibliographystyle{acl_natbib}
\bibliography{references}

\end{document}

%% file: macros.tex
\usepackage{amsmath}
\usepackage{amssymb}
\usepackage{amsthm}
\usepackage{xcolor}
\usepackage{paralist}
\usepackage{thm-restate}

\usepackage{mathtools}
\usepackage{multirow}

\DeclarePairedDelimiter\abs{\lvert}{\rvert}%
\DeclarePairedDelimiter\norm{\lVert}{\rVert}%

\usepackage{stmaryrd}
\DeclarePairedDelimiter{\den}{\llbracket}{\rrbracket}

\makeatletter
\let\oldabs\abs
\def\abs{\@ifstar{\oldabs}{\oldabs*}}
\let\oldnorm\norm
\def\norm{\@ifstar{\oldnorm}{\oldnorm*}}
\makeatother

\usepackage{xpatch}
\makeatletter
\AtBeginDocument{\xpatchcmd{\@thm}{\thm@headpunct{.}}{\thm@headpunct{}}{}{}}
\makeatother

\newtheorem{proposition}{Theorem}
\newtheorem{lemma}{Lemma}
\newtheorem{corollary}{Corollary}[proposition]
\newtheorem{lemcorollary}{Corollary}[lemma]
\theoremstyle{definition}
\newtheorem{definition}{Definition}

\def\Snospace~{\S{}}

\newcommand{\femb}{\phi}
\newcommand{\fact}{f}
\newcommand{\fattn}{s}

\newcommand{\AC}{\mathsf{AC}}
\newcommand{\NC}{\mathsf{NC}}
\newcommand{\TC}{\mathsf{TC}}
\newcommand\trans{\mathsf{AHAT}}

\newcommand{\poly}{\mathrm{poly}}
\renewcommand{\u}{\emph{u}}

\newcommand{\draftcomment}[3]{{\textcolor{#3}{[#1]#2}}}
\renewcommand{\draftcomment}[3]{}  

\usepackage{listings}

\lstdefinelanguage{RASP}
{
  morekeywords={select, selector_width, aggregate}
}

\definecolor{OliveGreen}{RGB}{0,102,51}
\newcommand\maj{\textsc{maj}}

\renewcommand{\O}{\mathrm{O}}

\newcommand{\new}[1]{{\color{blue} #1}}
\renewcommand{\new}[1]{#1}

%% file: appendix.tex
\appendix

\section{Float Division} \label{sec:floats}

Let $/$ be truncated division between integers.
We divide a float by an integer $p$ by defining an approximate multiplicative inverse $p^{-1}$. The numerator is $2^{\abs{p}} / p$ and the denominator is $2^{\abs{p}}$.
For division by a float $p, q$, we simply apply the integer approach and then multiply by $q$. This yields numerator $2^{\abs{p}} / p \cdot q$ and denominator $2^{\abs{p}}$.

\new{The fact that float division is defined in terms of integer multiplication and division implies it is size-preserving and can be simulated in $\TC^0$, which we use in \autoref{sec:simulation}.}

\section{Justifying Size Preservation} \label{sec:justifying-l}

We justify that feedforward neural networks are \new{size-preserving over floats}. Feedforward neural networks are made up of a fixed (with respect to $n$) number of addition, multiplication, division, ReLU, and square root (for layer norm) operations.
Therefore, it suffices to show that these operations are all in $\mathcal S(\mathbb F)$.

For addition, the numerator is
\begin{equation*}
    \leq p_1 q_2 + p_2 q_1 \leq 2 p_\mathrm{max} q_\mathrm{max} ,
\end{equation*}
which has size $\leq \log 2 + \abs{p_\mathrm{max}} + \abs{q_\mathrm{max}} \leq 2 (\abs{p_\mathrm{max}} + \abs{q_\mathrm{max}})$ for large enough input size.

For multiplication, the numerator is just $p_1 \cdot p_2$, which has size $\leq 2 \abs{p_\mathrm{max}}$. Let the denominators be $q_1 = 2^{k_1}$ and $q_2 = 2^{k_2}$. Then the denominator is $2^{k_1 + k_2}$, which has size $\leq 2 \abs{q_\mathrm{max}}$.

Division can be analyzed in terms of the approximate multiplicative inverse (\autoref{sec:floats}).\footnote{The exact multiplicative inverse $\langle p, q \rangle \mapsto \langle q, p \rangle$ over unconstrained rationals is also size-preserving. Thus, neural networks are size preserving over both floats and rationals.} Its numerator has size $\leq \abs{p} + 1 + \abs{q} \leq 2 (\abs{p} + \abs{q})$ for large enough input size. The denominator has size $\leq \abs{p} + 1 \leq 2 \abs{p}$ for large enough input size.

Size preservation is trivially satisfied for ReLU, which cannot expand the size of the input.

To make layer norm work, we just need to analyze square root, which we will define in a truncated fashion over integers. The square root of a rational, then, simple takes the square root of $p$ and $q$. We have that $\abs{\sqrt{p}} \leq \abs{p}$ and analogously for $q$.

\section{Resource-Bounded Transformers} \label{sec:resource-bounded}

Size preservation is one way to characterize \new{the constraints on transformers' internal functions}; a slightly different perspective is to fix $\phi$ and analyze how the language recognition abilities of the transformer change depending on the computational resources allotted to each \new{$s_{\ell,h}$ and $f_\ell$}. In this section, we derive an alternate universality theorem \new{in terms of time complexity classes}. We will show that as long as $\phi$ is powerful enough, \new{such transformers have equivalent time complexity to their activation functions.}


Recall that a transformer is a tuple \new{$\langle \Sigma, \mathbb{D}, \alpha, L, H, \phi, s_{\ell,h}, f_\ell \rangle$.} In contrast to $\trans(\mathbb D)$ (cf.~\autoref{def:trans-class}), we will now work with \new{a different class of transformer languages} \new{$\trans(\mathbb D, T(m))$ We will allow the embedding functions to be linear in the sequence length, and explore the effect of varying the complexity of the other internal functions}.
\new{Let $\mathsf{FTIME}(T(m))$ be the set of functions computable by a Turing machine in $T(m)$ time.\footnote{We write $\mathrm{FTIME}(m)$ instead of the conventional $\mathrm{FTIME}(n)$ to avoid confusion with the sequence length $n$.}
\begin{definition}
Let \new{$\trans(\mathbb D, T(n))$} be \new{the class of languages $\mathcal L \subseteq \Sigma^*$ such that there exists a transformer $\langle \Sigma, \mathbb{D}, \alpha, L, H, \phi, s_{\ell,h}, f_\ell \rangle$ that recognizes $\mathcal L$, where $\phi$ runs in time linear in the sequence length $n$, and $s_{\ell,h}, f_\ell \in \textsf{FTIME}(T(m))$}.
\end{definition}}


\new{For any $T(m) \geqslant m$, we will show transformers $\trans(\mathbb D, T(m))$ have the complexity of their activation functions.} Formally:


\paragraph{Theorem 2 \; {\normalfont (Formal version of \autoref{thm:resource-bounded-informal})}}
\textit{
\new{For $\mathbb D \in \{ \mathbb F, \mathbb Q \}$ and $T(m) \geqslant m$,}
\begin{equation*}
    \new{\trans(\mathbb D, T(m)) \subseteq \mathsf{TIME}(T(m)) .}
\end{equation*}
}

\begin{proof}
\new{First, observe that $\trans(\mathbb D, T(m)) \subseteq \mathsf{TIME}(T(m))$, since the embedding function and saturated attention can be computed in time linear in the input sequence length, and the other \new{internal functions} can be computed in $\mathsf{FTIME}(T(m))$ by construction.}

We now show $\mathsf{TIME}(m) \subseteq \trans(\mathbb D, T(m))$.
We adopt a $1$-layer transformer construction, \new{and thus omit $\ell,h$ subscripts}.\new{We define three components of the embedding function $\phi : \Sigma \times \mathbb{N} \to \mathbb D^3$:}
\begin{align*}
    \femb(w_i, i)_1 &=
    \begin{cases}
        2^{i - 1} & \textrm{if} \; w_i = 1 \\
        0 & \textrm{otherwise}
    \end{cases} \\
    \new{\femb(w_i, i)_2} &= \new{i} \\
    \new{\femb(w_i, i)_3} &= \new{2^{\abs{i}} .}
\end{align*}
\new{Each of these components is computable in time linear in $n$.}
\new{Define three heads $b_{1,i}$, $b_{2,i}$, $b_{3,i}$. Without loss of generality, consider $b_{h,i}$ to act on $\femb(w_i, i)_h$ alone, rather than the full embeding vector. $b_{1,i}$ is defined as a uniform head, while $b_{2,i}$ and $b_{3,i}$ are computed with $s_h(v_i, v_j) = v_j$. Thus,}
\begin{align*}
    b_{1,i} &= \frac{1}{n} \sum_{w_j = 1} 2^{j - 1} \\
    \new{b_{2,i}} &= \new{n }\\
    \new{b_{3,i}} &= \new{2^{\abs{n}} .}
\end{align*}
\new{Finally, we discuss how to set $f$ to compute whether $w \in \mathcal L$.}
Let $p$ be the function that extracts the numerator of a float or rational number, which is computable in $\O(m)$ time on \new{float} of size $m$. \new{Within $f$, we compute $u = p(b_{1,i})$.} At this point, we proceed in two cases depending on the datatype $\mathbb D$:

\begin{compactenum}
    \item \textbf{Rationals:} If $\mathbb D = \mathbb Q$, then $u$ is the binary string $w$. \new{Any $\mathcal L \in \mathsf{TIME}(T(m))$ has an indicator function $\delta \in \mathsf{FTIME}(T(m))$, which we now apply to recognize whether $w \in \mathcal L$.}
    
    \item \textbf{Floats:} If $\mathbb D = \mathbb F$, then $u = 2^{\abs{n}} / n \cdot w$ as in \autoref{sec:floats}. \new{Therefore, in linear time, we compute
\begin{equation*}
    \frac{b_{2,i}}{b_{3,i}} \cdot u = \frac{n}{2^{\abs{n}}} \cdot \frac{2^{\abs{n}}w}{n} = w ,
\end{equation*}}
and feed $w$ through $\delta$ as in the $\mathbb D = \mathbb Q$ case.
\end{compactenum}
\new{So, $\mathsf{TIME}(T(m)) \subseteq \trans(\mathbb D, T(m))$.}
\end{proof}

\new{\section{Proof from \citet{angluin2021}} \label{sec:angluin-proofs}

The proof for \autoref{lem:angluin} largely follows the proof of a core lemma of \citet{angluin2021}.
We reproduce a slightly adapted version of their proof here, since their manuscript is not yet publicly available, and we wish for our paper to be self-contained.
\paragraph{Lemma 2} Any function $f : \{0, 1\}^c \to \{0, 1\}^d$ can be computed by a boolean circuit of depth $3$ and size at most $d(2^c + c + 1)$.
\begin{proof}
The idea of the proof is to define $d$ subcircuits of size at most $2^c + c + 1$ that compute the $d$ output bits of $f$ in parallel.
We will build a circuit that computes each output bit of $f$ according to its representation in disjunctive normal form (DNF).
We define a first layer of the circuit that computes the negation of each input, which takes $c$ gates. The second layer then computes the value of each DNF term by computing a conjunction ($\wedge$ gate) over the corresponding literals or negated literals. Note that a formula of $c$ variables has at most $2^c$ DNF terms. Finally, the third layer of the circuit computes a disjunction ($\vee$ gate) over the values of all terms, yielding the output of $f$, and adding a single gate.
In summary, we have shown how to compute each output bit with a circuit of size at most $2^c + c + 1$, which implies the full function $f$ can be computed by a circuit of size at most $d(2^c + c + 1)$.
\end{proof}
}

%% file: SatAttnTC0-v2.bbl
\begin{thebibliography}{26}
\expandafter\ifx\csname natexlab\endcsname\relax\def\natexlab#1{#1}\fi

\bibitem[{Alishahi et~al.(2020)Alishahi, Belinkov, Chrupa{\l}a, Hupkes, Pinter,
  and Sajjad}]{blackboxnlp-2020-blackboxnlp}
Afra Alishahi, Yonatan Belinkov, Grzegorz Chrupa{\l}a, Dieuwke Hupkes, Yuval
  Pinter, and Hassan Sajjad, editors. 2020.
\newblock \href {https://www.aclweb.org/anthology/2020.blackboxnlp-1.0}
  {\emph{Proceedings of the Third BlackboxNLP Workshop on Analyzing and
  Interpreting Neural Networks for NLP}}. Association for Computational
  Linguistics, Online.

\bibitem[{Arora and Barak(2009)}]{arora2009computational}
Sanjeev Arora and Boaz Barak. 2009.
\newblock \href
  {https://books.google.com/books/about/Computational_Complexity.html?id=8Wjqvsoo48MC}
  {\emph{Computational Complexity: A Modern Approach}}.
\newblock Cambridge University Press.

\bibitem[{Ba et~al.(2016)Ba, Kiros, and Hinton}]{ba2016layer}
Jimmy Ba, Jamie~Ryan Kiros, and Geoffrey~E. Hinton. 2016.
\newblock Layer normalization.
\newblock \emph{ArXiv}, abs/1607.06450.

\bibitem[{Bhattamishra et~al.(2020)Bhattamishra, Ahuja, and
  Goyal}]{bhattamishra-etal-2020-ability}
Satwik Bhattamishra, Kabir Ahuja, and Navin Goyal. 2020.
\newblock \href {https://doi.org/10.18653/v1/2020.emnlp-main.576} {On the
  ability and limitations of transformers to recognize formal languages}.
\newblock In \emph{Proceedings of the 2020 Conference on Empirical Methods in
  Natural Language Processing (EMNLP)}, pages 7096--7116, Online. Association
  for Computational Linguistics.

\bibitem[{Cojocaru(2016)}]{cojocaru2016advanced}
Liliana Cojocaru. 2016.
\newblock \href
  {https://trepo.tuni.fi/bitstream/handle/10024/99577/978-952-03-0184-2.pdf}
  {\emph{Advanced Studies on the Complexity of Formal Languages}}.
\newblock Ph.D. thesis, University of Tampere.

\bibitem[{Devlin et~al.(2019)Devlin, Chang, Lee, and
  Toutanova}]{devlin-etal-2019-bert}
Jacob Devlin, Ming-Wei Chang, Kenton Lee, and Kristina Toutanova. 2019.
\newblock \href {https://doi.org/10.18653/v1/N19-1423} {{BERT}: Pre-training of
  deep bidirectional transformers for language understanding}.
\newblock In \emph{Proceedings of the 2019 Conference of the North {A}merican
  Chapter of the Association for Computational Linguistics: Human Language
  Technologies, Volume 1 (Long and Short Papers)}, pages 4171--4186,
  Minneapolis, Minnesota. Association for Computational Linguistics.

\bibitem[{van Emde~Boas(1991)}]{chapter1}
Peter van Emde~Boas. 1991.
\newblock \emph{Machine Models and Simulations}, chapter~1. MIT Press,
  Cambridge, MA, USA.

\bibitem[{Furst et~al.(1981)Furst, Saxe, and Sipser}]{furst81parity}
Merrick Furst, James~B. Saxe, and Michael Sipser. 1981.
\newblock \href {https://doi.org/10.1109/SFCS.1981.35} {Parity, circuits, and
  the polynomial-time hierarchy}.
\newblock In \emph{Proceedings of the 22nd Annual Symposium on Foundations of
  Computer Science}, SFCS '81, page 260–270, USA. IEEE Computer Society.

\bibitem[{Goldstein(1973)}]{goldstein1973history}
Larry~J Goldstein. 1973.
\newblock A history of the prime number theorem.
\newblock \emph{The American Mathematical Monthly}, 80(6):599--615.

\bibitem[{Hahn(2020)}]{hahn-2020-theoretical}
Michael Hahn. 2020.
\newblock \href {https://www.aclweb.org/anthology/2020.tacl-1.11} {Theoretical
  limitations of self-attention in neural sequence models}.
\newblock \emph{Transactions of the Association for Computational Linguistics},
  8:156--171.

\bibitem[{Hao et~al.(2022)Hao, Angluin, and Frank}]{angluin2021}
Yiding Hao, Dana Angluin, and Robert Frank. 2022.
\newblock Hard attention transformers and constant depth circuits.
\newblock Unpublished manuscript.

\bibitem[{Hewitt et~al.(2020)Hewitt, Hahn, Ganguli, Liang, and
  Manning}]{hewitt-etal-2020-rnns}
John Hewitt, Michael Hahn, Surya Ganguli, Percy Liang, and Christopher~D.
  Manning. 2020.
\newblock \href {https://doi.org/10.18653/v1/2020.emnlp-main.156} {{RNN}s can
  generate bounded hierarchical languages with optimal memory}.
\newblock In \emph{Proceedings of the 2020 Conference on Empirical Methods in
  Natural Language Processing (EMNLP)}, pages 1978--2010, Online. Association
  for Computational Linguistics.

\bibitem[{Johnson(1991)}]{chapter2}
David~S. Johnson. 1991.
\newblock \emph{A Catalog of Complexity Classes}, chapter~2. MIT Press,
  Cambridge, MA, USA.

\bibitem[{Kayal(2015)}]{tc0notes}
Neeraj Kayal. 2015.
\newblock \href
  {https://www.csa.iisc.ac.in/~chandan/courses/arithmetic_circuits/notes/lec5.pdf}
  {Lecture notes for topics in complexity theory}.

\bibitem[{Mahajan(2007)}]{mahajan2007polynomial}
Meena Mahajan. 2007.
\newblock \href
  {https://www.uni-ulm.de/fileadmin/website_uni_ulm/iui.inst.190/Mitarbeiter/toran/beatcs/column91.ps}
  {Polynomial size log depth circuits: Between {$\NC^1$} and {$\AC^1$}.}
\newblock \emph{Bulletin of the EATCS}, 91:42--56.

\bibitem[{Merrill(2019)}]{merrill-2019-sequential}
William Merrill. 2019.
\newblock \href {https://doi.org/10.18653/v1/W19-3901} {Sequential neural
  networks as automata}.
\newblock In \emph{Proceedings of the Workshop on Deep Learning and Formal
  Languages: Building Bridges}, pages 1--13, Florence. Association for
  Computational Linguistics.

\bibitem[{Merrill(2021)}]{merrill2021formal}
William Merrill. 2021.
\newblock \href {https://arxiv.org/pdf/2102.10094.pdf} {Formal language theory
  meets modern {NLP}}.
\newblock \emph{ArXiv}, abs/2102.10094.

\bibitem[{Merrill et~al.(2021)Merrill, Ramanujan, Goldberg, Schwartz, and
  Smith}]{merrill2020parameter}
William Merrill, Vivek Ramanujan, Yoav Goldberg, Roy Schwartz, and Noah~A.
  Smith. 2021.
\newblock \href {https://doi.org/10.18653/v1/2021.emnlp-main.133} {Effects of
  parameter norm growth during transformer training: Inductive bias from
  gradient descent}.
\newblock In \emph{Proceedings of the 2021 Conference on Empirical Methods in
  Natural Language Processing}, pages 1766--1781, Online and Punta Cana,
  Dominican Republic. Association for Computational Linguistics.

\bibitem[{Peng et~al.(2018)Peng, Schwartz, Thomson, and
  Smith}]{peng-etal-2018-rational}
Hao Peng, Roy Schwartz, Sam Thomson, and Noah~A. Smith. 2018.
\newblock \href {https://doi.org/10.18653/v1/D18-1152} {Rational recurrences}.
\newblock In \emph{Proceedings of the 2018 Conference on Empirical Methods in
  Natural Language Processing}, pages 1203--1214, Brussels, Belgium.
  Association for Computational Linguistics.

\bibitem[{Pérez et~al.(2019)Pérez, Marinković, and Barceló}]{perez2018on}
Jorge Pérez, Javier Marinković, and Pablo Barceló. 2019.
\newblock \href {https://openreview.net/forum?id=HyGBdo0qFm} {On the {Turing}
  completeness of modern neural network architectures}.
\newblock In \emph{International Conference on Learning Representations}.

\bibitem[{Vaswani et~al.(2017)Vaswani, Shazeer, Parmar, Uszkoreit, Jones,
  Gomez, Kaiser, and Polosukhin}]{vaswani2017attention}
Ashish Vaswani, Noam Shazeer, Niki Parmar, Jakob Uszkoreit, Llion Jones,
  Aidan~N Gomez, \L~ukasz Kaiser, and Illia Polosukhin. 2017.
\newblock \href
  {https://proceedings.neurips.cc/paper/2017/file/3f5ee243547dee91fbd053c1c4a845aa-Paper.pdf}
  {Attention is all you need}.
\newblock In \emph{Advances in Neural Information Processing Systems},
  volume~30. Curran Associates, Inc.

\bibitem[{Weiss et~al.(2018)Weiss, Goldberg, and
  Yahav}]{weiss-etal-2018-practical}
Gail Weiss, Yoav Goldberg, and Eran Yahav. 2018.
\newblock \href {https://doi.org/10.18653/v1/P18-2117} {On the practical
  computational power of finite precision {RNN}s for language recognition}.
\newblock In \emph{Proceedings of the 56th Annual Meeting of the Association
  for Computational Linguistics (Volume 2: Short Papers)}, pages 740--745,
  Melbourne, Australia. Association for Computational Linguistics.

\bibitem[{Weiss et~al.(2021)Weiss, Goldberg, and Yahav}]{weiss2021thinking}
Gail Weiss, Yoav Goldberg, and Eran Yahav. 2021.
\newblock \href {https://arxiv.org/pdf/2106.06981.pdf} {Thinking like
  transformers}.
\newblock \emph{ArXiv}, abs/2106.06981.

\bibitem[{Yao(1990)}]{yao1990acc}
Andrew C.-C. Yao. 1990.
\newblock \href {https://doi.org/10.1109/FSCS.1990.89583} {On {ACC} and
  threshold circuits}.
\newblock In \emph{Proceedings of the 31st Annual Symposium on Foundations of
  Computer Science}, pages 619--627 vol.2.

\bibitem[{Yao et~al.(2021)Yao, Peng, Papadimitriou, and
  Narasimhan}]{yao2021selfattention}
Shunyu Yao, Binghui Peng, Christos Papadimitriou, and Karthik Narasimhan. 2021.
\newblock \href {https://doi.org/10.18653/v1/2021.acl-long.292} {Self-attention
  networks can process bounded hierarchical languages}.
\newblock In \emph{Proceedings of the 59th Annual Meeting of the Association
  for Computational Linguistics and the 11th International Joint Conference on
  Natural Language Processing (Volume 1: Long Papers)}, pages 3770--3785,
  Online. Association for Computational Linguistics.

\bibitem[{Yun et~al.(2020)Yun, Bhojanapalli, Rawat, Reddi, and
  Kumar}]{Yun2020Are}
Chulhee Yun, Srinadh Bhojanapalli, Ankit~Singh Rawat, Sashank Reddi, and Sanjiv
  Kumar. 2020.
\newblock \href {https://openreview.net/forum?id=ByxRM0Ntvr} {Are transformers
  universal approximators of sequence-to-sequence functions?}
\newblock In \emph{International Conference on Learning Representations}.

\end{thebibliography}
